\documentclass[11pt,twoside]{article}
\usepackage{fancyhdr}
\usepackage[colorlinks,citecolor=blue,urlcolor=blue,linkcolor=blue,bookmarks=false]{hyperref}
\usepackage{amsfonts,epsfig,graphicx}
\usepackage{afterpage}
\usepackage{amsmath,amssymb,amsthm} 
\usepackage{fullpage}
\usepackage[T1]{fontenc} 
\usepackage{epsf} 
\usepackage{graphics} 
\usepackage{amsfonts,amsmath}
\usepackage[sort,numbers]{natbib} 
\usepackage{psfrag,xspace}
\usepackage{color,etoolbox}
\usepackage{tikz}

\newcommand{\rot}{\mathfrak{R}_{\text{OT}}}
\newcommand{\rvalid}{\mathfrak{R}_{\text{valid}}}
\newcommand{\lot}{\mathcal{L}_{\text{OT}}}
\newcommand{\lots}{\mathcal{L}^\text{S}_{\text{OT}}}

\newcommand{\lvalid}{\mathcal{L}_{\text{valid}}}
\newcommand{\lvalidp}{\mathcal{L}_{\text{valid,p}}}
\newcommand{\lvalids}{\mathcal{L}^S_{\text{valid}}}

\newtheorem{theorem}{Theorem} 
 
\newtheorem{lemma}{Lemma}
\newtheorem{proposition}{Proposition}
\newtheorem{corollary}{Corollary}

\DeclareMathOperator{\TV}{TV}

\DeclareMathOperator{\sgn}{sgn}

\DeclareMathOperator*{\argmin}{arg\,min}

\newcommand{\OT}{\mathrm{OT}}

\newcommand{\targetclass}{\mathfrak{N}}
\newcommand{\validminimaxrisk}{\mathfrak{M}_{\text{valid}}}
\newcommand{\validsminimaxrisk}{\mathfrak{M}^S_{\text{valid}}}

\newcommand{\otminimaxrisk}{\mathfrak{M}_{\text{OT}}}
\newcommand{\wtwominimaxrisk}{\mathfrak M_{W_2}}

\usetikzlibrary{arrows.meta,positioning}

\begin{document}

\begin{center} {\LARGE{\bf{The Fundamental Limits of \\
\vspace{.3cm}

Valid Transport Map Estimation}}}
\\

\vspace*{.3in}

{\large{
\begin{tabular}{ccccc}
Sivaraman Balakrishnan$^{\dagger\ddagger}$ \\
\end{tabular}

\vspace*{.1in}

\begin{tabular}{ccc}
Department of Statistics and Data Science$^{\dagger}$ \\
Machine Learning Department$^{\ddagger}$ \\
\end{tabular}

\begin{tabular}{c}
Carnegie Mellon University, \\
Pittsburgh, PA 15213.
\end{tabular}

\vspace*{.2in}
}}

\vspace*{.2in}

\today
\vspace*{.2in}
\begin{abstract}

Many modern generative modeling methods, including diffusion models, normalizing flows, and flow matching, estimate transport maps or plans between distributions without explicitly targeting an \emph{optimal} transport (OT) map. In applications like generative modeling, the transport cost itself is irrelevant, and this makes it natural to target maps which are more tractable from either a statistical or computational standpoint. In this short note, we formalize the task of estimating \emph{any} valid transport map in a rigorous minimax framework. 
One consequence of this framing is that it yields sample complexity lower bounds for any method whose learned object is evaluated as a transport map or plan, including flow matching and diffusion-based generative models, in settings where direct analysis would be challenging due to the analytic complexity of the methods and their target maps.
We observe that, under standard, though strong, stability assumptions from the OT literature, estimating any valid transport map is statistically as hard as estimating the OT map. 
We complement these results with some examples showing that when these stability assumptions fail, alternative transport maps can be learned substantially more accurately than the OT map. Our minimax framing provides a rigorous foundation for understanding the statistical limits of modern transport-based generative methods and clarifies when targeting sub-optimal maps can provide real statistical advantages.
\end{abstract}
\end{center}

\section{Introduction}
Learning transformations between probability distributions is a central problem in modern statistics and machine learning. Such transformations 
provide a mechanism for mapping a simple reference distribution (e.g., a Gaussian) to a complex target distribution. A wide variety of modern generative modeling methods can be interpreted through this lens. For example, normalizing flows explicitly parameterize invertible transport maps between distributions \cite{rezende2015variational,dinh2017density,papamakarios2021normalizing}, while diffusion models and score-based generative models construct stochastic processes that implicitly define transformations between a reference distribution and the target distribution \cite{sohl2015deep,ho2020denoising,song2021score}. More recent approaches such as flow matching and rectified flows also explicitly frame generative modeling as learning transport maps between distributions \cite{lipman2023flowmatching,liu2022rectified,albergo2023stochastic}. These methods have demonstrated remarkable empirical success across a wide range of generative modeling tasks.

A classical mathematical formulation of the problem of transforming one distribution to another is given by optimal transport (OT) \cite{villani2003topics,villani2009optimal}. In this framework, the optimal transport map between two distributions is defined as the map that pushes one distribution to the other while minimizing a specified transport cost. 
Optimal transport has become an important tool in machine learning and statistics, with applications ranging from generative modeling and domain adaptation to distributional robustness and fairness~\cite{peyre2019computational}.
Part of the success of OT is that it pins down a clear target of inference, enabling us to delve deeper into statistical and computational issues.
At the same time, computing or estimating optimal transport maps from samples can be challenging, particularly in high dimensions.

In many applications, the transport cost itself is irrelevant. The goal is simply to produce samples from a target distribution or to transform data between domains; any valid transport map that pushes the source distribution to the target distribution suffices. This observation motivates the widespread use of methods that do not explicitly target the optimal transport map, but instead learn alternative transport mechanisms. Indeed, many modern generative modeling methods -- including diffusion models, flow matching, and various triangular or autoregressive transport maps -- can be interpreted as learning transport maps or stochastic transport plans that are generally sub-optimal with respect to the classical OT objective.

This raises a natural question: is learning an arbitrary transport map statistically easier than learning the optimal transport map? A common intuition is that this should be the case. Optimal transport maps can be sensitive to perturbations in the underlying distributions, and estimating them from samples often requires strong regularity assumptions \cite{caffarelli1992regularity,figalli2017monge}. In contrast, modern generative modeling methods often appear empirically stable and effective even in settings where optimal transport maps may be difficult to estimate. These observations have led to a belief that learning a “good enough’’ transport map may be fundamentally easier, both computationally and \emph{statistically}, than learning the optimal one.

Despite this intuition, the statistical difficulty of estimating general transport maps has not been systematically studied. While there is a growing literature on the statistical estimation of optimal transport maps \cite{manole2024plugin,deb2021,hutter2021}, much less is known about the fundamental limits of estimating arbitrary transport maps between distributions. Moreover, for many modern generative modeling methods the target map is analytically complex, making it difficult to study their statistical properties directly.

In this paper, we introduce a simple minimax framework for studying the statistical limits of transport map estimation. Instead of focusing on the optimal transport map, we consider the broader problem of estimating a measurable map that transports one distribution to another. We measure performance using natural notions of risk that quantify the discrepancy between the estimated transport map and the set of valid transport maps.

Our main results show that the task of estimating a valid transport map lends itself quite naturally to statistical analysis. Intuitively, any successful transport map estimator must at least induce an accurate estimate of the target distribution in Wasserstein distance, and the difficulty of this latter task serves as a fundamental limit for the estimation of a transport map. Since the optimal transport map is itself a valid transport map, estimating a valid map is never harder than estimating the optimal one. Stability assumptions commonly imposed in the literature on OT map estimation ensure that the OT map varies smoothly with the underlying distributions. Under these stability assumptions, the risk of OT map estimation is upper bounded by the risk of estimating the target distribution in the Wasserstein sense. We show in this paper that, under stability assumptions, estimating a valid transport map and estimating the optimal transport map have the same statistical difficulty. 

At the same time, our results also show that the equivalence between optimal and sub-optimal transport maps can fail in the absence of stability. We provide simple examples in which the optimal transport map is either moderately or highly unstable and statistically difficult to estimate, while alternative transport maps can be learned at substantially faster rates. This separation demonstrates that the potential advantages of modern transport-based generative methods arise precisely in regimes where stability assumptions break down.

Our framing immediately yields new lower bounds for a range of modern generative modeling methods. 
In particular, our results imply sample complexity lower bounds for formulations of methods such as flow matching and diffusion-based generative models whose learned object is evaluated as a transport map or plan.
Directly analyzing the statistical limits of these methods would be challenging due to the analytic complexity of the transformations they learn. By instead studying the broader problem of valid transport map estimation, we obtain these bounds in a unified and conceptually simple manner.

Taken together, our results provide a principled statistical perspective on the problem of learning transport maps. The minimax framework introduced here grounds the discussion of statistical limits in modern transport-based generative modeling and helps identify settings in which targeting sub-optimal transport maps can provide genuine statistical advantages.

\subsection{Notation}
For a metric space $E$, we write $\mathcal P(E)$ for the set of Borel probability measures on $E$. For $p\geq 1$, we write $\mathcal P_p(\mathbb R^d)$ for the set of probability measures on $\mathbb R^d$ with finite $p$th moment. Given probability measures $\mu,\nu\in\mathcal P(\mathbb R^d)$, we write $\Pi(\mu,\nu)$ for the set of couplings of $\mu$ and $\nu$. The $p$-Wasserstein distance is denoted by $W_p$, so that $W_p^p(\mu,\nu)
=
\inf_{\pi\in\Pi(\mu,\nu)}
\int \|x-y\|_2^p\,d\pi(x,y).$
We use $\rho(P,Q)$ for the Hellinger affinity,
$\rho(P,Q)
:=
\int \sqrt{dP\,dQ},$
and $H^2(P,Q):=1-\rho(P,Q)$ for the squared Hellinger distance. Finally, for nonnegative sequences $a_n,b_n$, we write $a_n\lesssim b_n$ if $a_n\leq Cb_n$ for a constant $C$ independent of $n$, and write $a_n\asymp b_n$ if both $a_n\lesssim b_n$ and $b_n\lesssim a_n$. We write $a_n\ll b_n$ when the ratio $a_n/b_n$ is chosen sufficiently small. 

\section{Background}
Inspired by methods in generative modeling and by the growing literature in statistical optimal transport, we focus on the following setup. 
We have a \emph{known} source distribution $\mu$, and i.i.d. samples $Y_1, \ldots, Y_n$ from an unknown target distribution $\nu$ both supported on $\mathbb{R}^d$.  Our goal is to estimate a \emph{transport map} $T$. Informally, a transport map $T:\mathbb{R}^d \to \mathbb{R}^d$ between $\mu$ and $\nu$ is a map such that for $X \sim \mu$, we have that $T(X) \sim \nu$. More formally, a transport map between $\mu$ and $\nu$
is a Borel-measurable function such that $T_{\#} \mu := \mu(T^{-1}(\cdot)) = \nu$. 

If $\mu$ is non-atomic 
there is always at least one valid transport map between $\mu$ and $\nu$\footnote{Throughout, whenever we study deterministic maps, we assume either that $\mu$ is non-atomic, 
so that valid maps to all target distributions under consideration exist, or else we use the convention that the infimum over an empty set is $+\infty$.}.
There are often many transport maps between $\mu$ and $\nu$. 
Given this non-uniqueness, the literature on \emph{optimal transport} has focused on a particular choice of transport map. 
Given a cost $c: \mathbb{R}^d \times \mathbb{R}^d \to \mathbb{R}$ (most popularly, the squared Euclidean distance), the OT map is defined as:
\begin{align*}
T_0 := \underset{T: T_{\#} \mu = \nu}{\text{arg min}}~\int c(X,T(X)) d\mu(X). 
\end{align*}
Brenier's celebrated polar factorization theorem~\cite{brenier1991} shows that, for the squared Euclidean cost under appropriate regularity conditions, any valid transport map can be decomposed into the composition of 
the OT map with a measure-preserving map (i.e. a map which maps $\mu$ onto itself). 
These facts highlight that the OT map is a natural parsimonious choice to target.

On the other hand, methods like diffusion models learn stochastic maps (i.e. couplings) between the source and target. A coupling $\pi$ is a joint distribution on $\mathbb{R}^d \times \mathbb{R}^d$ with first marginal $\mu$ and second marginal $\nu$, and it prescribes a randomized transport plan between $\mu$ and $\nu$. More formally, we denote the set of couplings as $\Pi(\mu,\nu)$, i.e.
\begin{align*}
\Pi(\mu,\nu) = \{ \pi \in \mathcal{P}(\mathbb{R}^d \times \mathbb{R}^d): \pi( \cdot \times \mathbb{R}^d) = \mu, \pi( \mathbb{R}^d \times \cdot) = \nu \}.
\end{align*} 
Unlike transport maps, couplings always exist. Once again a canonical choice is an OT coupling $\pi_0$ (i.e. a coupling that minimizes a prescribed cost on average). 
For costs which satisfy certain regularity conditions (the squared Euclidean cost being an example), when the source $\mu$ is absolutely continuous with respect to the Lebesgue measure, the OT coupling is unique and it is concentrated on a graph (i.e. corresponds directly to a unique OT map).

\subsection{The Minimax Framework}
We first briefly review the minimax framework for OT map estimation which has been used in a series of prior works~\cite{manole2024plugin,deb2021,hutter2021,balakrishnan2025statistical,balakrishnan2026stability}.
The minimax framework provides a 
foundation for comparing estimators of the OT map based on their sample efficiency. This
framework also allows the statistician to characterize the fundamental limits of estimation
from random samples and consequently to reason about optimal estimators. 
The literature on OT maps has largely focused on the squared Euclidean risk. Concretely, given an 
estimator $\widehat{T}$, we can evaluate it by comparing it to the OT map $T_0$ via its loss:
\begin{align*}
\lot(\widehat{T}) := \int \|\widehat{T}(X) - T_0(X)\|_2^2 d\mu(X).
\end{align*}
Though we primarily focus on the squared Euclidean loss, the results we develop in this paper also naturally extend to general $L_p$ losses for $p > 1$. 
To evaluate stochastic methods it makes sense to generalize this loss to:
\begin{align*}
\lots(\widehat{\pi}) := \int W_2^2(\widehat{\pi}(\cdot | X), \pi_0( \cdot | X)) d\mu(X).
\end{align*}
We focus the main paper on deterministic transport maps and briefly address the extension of our results to stochastic maps in Appendix~\ref{app:stochastic}.
The loss is a random variable (since the estimator $\widehat{T}$ is sample dependent), and it is standard to instead focus on studying the risk
(its expected value):
\begin{align*}
\rot(\widehat{T}) = \mathbb{E} \lot(\widehat{T}). 
\end{align*}
While this loss and risk are reasonable as a benchmark for estimating the OT map, they are of course not reasonable measures to study procedures like rectified flow or normalizing flows, which explicitly target alternative non-optimal maps.
In this paper, we focus instead on the following loss and risk:
\begin{align*}
\lvalid(\widehat{T}) = \left[\ \inf_{T_{\#} \mu = \nu} \int \|\widehat{T}(X) - T(X)\|_2^2 d\mu(X)\right],~~~\rvalid(\widehat{T}) = \mathbb{E} \lvalid(\widehat{T})
\end{align*}
This loss measures how close $\widehat{T}$ is to the set of valid transport maps, rather than to a specific canonical choice.
Equivalently, $\lvalid$ is the squared $L^2(\mu)$ distance from $\widehat T$ to the set ${T : T_{\#}\mu = \nu}$.
It is clear that:
\begin{align*}
\lot(\widehat{T}) \geq \lvalid(\widehat{T}),~~~\text{and}~~~
\rot(\widehat{T}) \geq \rvalid(\widehat{T}),
\end{align*}
and part of the purpose of this note is to understand the gap between these quantities. We will show both
that $\rvalid$ is amenable to analysis and also does provide some insights into when targeting suboptimal maps can be useful.

With a notion of risk in place, the main object of study is often the \emph{minimax} risk, i.e.
\begin{align*}
\validminimaxrisk(\targetclass) = \inf_{\widehat{T}} \sup_{\nu \in \targetclass} \rvalid(\widehat{T})~~\text{and}~~\otminimaxrisk(\targetclass) = \inf_{\widehat{T}} \sup_{\nu \in \targetclass} \rot(\widehat{T}),
\end{align*}
where $\targetclass$ is some collection of potential target distributions\footnote{Throughout the paper, we suppress the dependence of these risks on the known source distribution $\mu$ and on the sample size $n$.}.
Prior work~\citep{hutter2021, divol2022a} has used variants of this definition, where rather than restrict the class of target distributions $\targetclass$, they restrict the source distribution $\mu$ and the optimal transport map $T_0$ between $\mu$ and $\nu$ to satisfy regularity conditions. 
Other work~\citep{manole2024plugin,balakrishnan2026stability} has used both regularity of the transport map and direct restrictions on the class of target distributions.

\subsection{Statistical Estimation of OT Maps}
There is by now a substantial literature on the statistical estimation of OT maps from samples; see, for instance, \cite{manole2024plugin,deb2021,hutter2021} and the recent survey \cite{balakrishnan2025statistical}. Some of this literature studies estimators which first estimate the underlying distributions and then plug these estimates into the OT problem. The analysis of these estimators often rests on a strong regularity condition on the population OT map. A representative assumption is that the OT map from $\mu$ to $\nu$ is of the form
$T_0 = \nabla \varphi_0,$
where the so-called \emph{Brenier potential} $\varphi_0$ is both smooth and strongly convex. For instance, when $\varphi_0$ is twice differentiable, one may assume that for some constants $0 < \alpha \leq \beta < \infty$,
\begin{align*}
\alpha I \preceq \nabla^2 \varphi_0(x) \preceq \beta I, \qquad x \in \mathbb{R}^d.
\end{align*}
This condition is restrictive, but it appears naturally in several classical settings, for example through Caffarelli-type regularity and contraction results for log-concave measures \cite{caffarelli1992regularity,figalli2017monge}.

A key consequence of this regularity condition is that under it the OT map is stable with respect to perturbations of the input distributions. We recall one version of this fact (see~Theorem 3 in~\cite{balakrishnan2026stability}). Given probability measures $\mu$ and $\widehat{\nu}$, where $\mu$ is absolutely continuous with respect to the Lebesgue measure, let $\widehat{T}_{\OT}$ denote 
the OT map from $\mu$ to $\widehat{\nu}$.
\begin{lemma}[Stability of the OT map]
\label{lem:ot-stability}
Suppose that $\nu = (\nabla \varphi_0)_{\#}\mu$, where $\varphi_0$ is $\alpha$-strongly convex and $\beta$-smooth. Then, for any probability measure $\widehat{\nu}$ with finite second moments,
\begin{align*}
\int \|\widehat{T}_{\OT}(x)-T_0(x)\|_2^2 d\mu(x)
\leq \frac{\beta}{\alpha} W^2_2(\widehat{\nu},\nu).
\end{align*}
\end{lemma}

\noindent When $\mu$ is absolutely continuous every distribution $\widehat{\nu}$ corresponds to a transport map $\widehat{T}_{\OT}$ which pushes $\mu$ onto $\widehat{\nu}$. Consequently, this extraordinary stability shows that, under the above regularity assumptions, when $\alpha, \beta$ are fixed estimating the OT map is no harder than estimating the target distribution in Wasserstein distance.

\section{Main Results}
With the necessary background in place we now proceed to discuss our main results. In Section~\ref{sec:structural} we discuss some basic properties of the loss $\lvalid$ which show that it is often amenable to direct analysis. In Section~\ref{sec:implications} we discuss the implications of this basic result, providing both new lower bounds as well as enabling a direct comparison with existing results on OT map estimation. 
In Section~\ref{sec:separation} we prove a pair of minimax lower bounds, which highlight the possibility of separations between the OT and valid map minimax risks. We focus throughout on the multivariate case when $d \geq 2$ and defer treatment of the special one-dimensional case to Appendix~\ref{sec:oned}. We defer detailed proofs of various
claims to Appendix~\ref{app:proofs}.

\subsection{Structural Results}
\label{sec:structural}
We first give a simple structural result which shows that the projection risk defined above is closely related to the Wasserstein distance between the distribution induced by the estimator and the target distribution.

\begin{lemma}[Projection risk and Wasserstein distance]
\label{lem:main}
Let $\mu,\nu \in \mathcal{P}_p(\mathbb{R}^d)$, and let $\widehat{T}:\mathbb{R}^d \to \mathbb{R}^d$ be a measurable map such that
\[
\widehat{\nu}:=\widehat{T}_{\#}\mu \in \mathcal{P}_p(\mathbb{R}^d).
\]
Then
\begin{align}
\label{eq:valid-risk-lower}
\lvalidp(\widehat{T}) := \inf_{T_{\#}\mu=\nu}
\int \|\widehat{T}(x)-T(x)\|_2^p\,d\mu(x)
\geq
W_p^p(\widehat{\nu},\nu).
\end{align}
Moreover, suppose that there exists an OT map $S:\mathbb{R}^d \to \mathbb{R}^d$ transporting $\widehat{\nu}$ to $\nu$, i.e.
\[
S_{\#}\widehat{\nu}=\nu
\qquad\text{and}\qquad
\int \|z-S(z)\|_2^p\,d\widehat{\nu}(z)
=
W_p^p(\widehat{\nu},\nu).
\]
Then equality holds in \eqref{eq:valid-risk-lower}. 
\end{lemma}
\noindent For $p>1$, there is an OT map $S:\mathbb{R}^d \to \mathbb{R}^d$ transporting $\widehat{\nu}$ to $\nu$ whenever $\widehat{\nu}$ is absolutely continuous with respect to the Lebesgue measure. 
\noindent Consequently, at least when the estimator induces an absolutely continuous distribution, the valid map loss $\lvalid$ is exactly the Wasserstein error of the induced distribution $\widehat{\nu}=\widehat{T}_{\#}\mu$. 

\subsection{Implications}
\label{sec:implications}
The structural result above has a useful consequence: estimating a valid transport map contains, as a subproblem, estimating the target distribution in Wasserstein distance. This observation is independent of optimality, convexity, smoothness, or the particular algorithm used to construct the map. It therefore gives immediate lower bounds for procedures such as rectified flow, flow matching, normalizing flows, and stochastic transport variants, whenever their output is evaluated as a map or plan from the source distribution to the target distribution.

Let $Y_1,\ldots,Y_n\sim \nu$ be i.i.d., and let $\targetclass\subseteq \mathcal P_2(\mathbb R^d)$ be a class of possible target distributions. Define the minimax squared-Wasserstein distribution-estimation risk
\begin{align*}
\wtwominimaxrisk(\targetclass):=
\inf_{\widehat\nu}
\sup_{\nu\in\targetclass}
\mathbb E_\nu W_2^2(\widehat\nu,\nu),
\end{align*}
where the infimum is over all measurable distribution estimators based on $Y_1,\ldots,Y_n$. The following result is an immediate consequence of Lemma~\ref{lem:main}.

\begin{lemma}[Distribution estimation is a subproblem]
\label{lem:wasserstein-lower-bound-valid map}
For any known source distribution $\mu\in\mathcal P_2(\mathbb R^d)$ and any target class $\targetclass\subseteq\mathcal P_2(\mathbb R^d)$,
\begin{align}
\label{eq:valid map-lower-bound-by-w2}
\validminimaxrisk(\targetclass)
\geq
\wtwominimaxrisk(\targetclass).
\end{align}
\end{lemma}

\paragraph{Consequences for rectified flow and related map estimators.}
Rectified flow and related estimators~\cite{lipman2023flowmatching,liu2022rectified,albergo2023stochastic}
output a data-dependent velocity field $\widehat v_t$ and then define a terminal map $\widehat T_{\rm RF}$ by solving
\begin{align*}
\frac{d}{dt}\widehat X_t = \widehat v_t(\widehat X_t),
\qquad
\widehat X_0=x,
\qquad
\widehat T_{\rm RF}(x):=\widehat X_1,
\end{align*}
whenever this ODE is well-posed. Once the terminal map is constructed, it is simply a map estimator in the sense above. Therefore Lemma~\ref{lem:wasserstein-lower-bound-valid map} gives
\begin{align}
\label{eq:rf-lower-bound-general}
\inf_{\widehat v}
\sup_{\nu\in\targetclass}
\mathbb E_\nu
\left[
\inf_{T_\#\mu=\nu}
\int \|\widehat T_{\rm RF}(x)-T(x)\|_2^2\,d\mu(x)
\right]
\geq
\wtwominimaxrisk(\targetclass),
\end{align}
where the infimum is over any class of data-dependent vector-field estimators whose terminal maps are measurable. This lower bound is algorithm-independent: it does not depend on how the velocity field is parameterized, optimized, or regularized. The same statement applies to flow matching, normalizing flows, triangular flows, and any other deterministic transport estimator after replacing $\widehat T_{\rm RF}$ by its terminal map. For stochastic diffusion-type estimators, Lemma~\ref{lem:kernel-risk-equals-w2-risk} gives the corresponding result for the stochastic projection loss.

\subsubsection*{A concrete smooth-density lower bound}
For illustration we now instantiate the preceding reduction in a canonical non-parametric setting using known minimax lower bounds for Wasserstein distribution estimation. Let $\Omega=[0,1]^d$, and for $s\geq0$, $L>0$, and $m>0$, define
\begin{align*}
\targetclass^s_{p',q}(L;m)
:=
\left\{
 f\in B^s_{p',q}(L):
 f\geq m,\ \int_\Omega f(x)\,dx=1
\right\},
\end{align*}
where $B^s_{p',q}(L)$ denotes the usual Besov ball on $\Omega$. Abusing notation slightly, we use the same notation for the corresponding collection of measures.
This is the smooth-density model studied by \citet{weed2019a}, and their results, together with Lemma~\ref{lem:wasserstein-lower-bound-valid map} immediately implies the following corollary: 
\begin{corollary}[Smooth-density lower bound for valid transport estimation]
\label{cor:smooth-density-valid map-lower-bound}
Fix $d\geq1$, $s\geq0$, $2\leq p'<\infty$, and $1\leq q\leq\infty$, and suppose the above class is nonempty. For every known source distribution $\mu\in\mathcal P_2(\mathbb R^d)$, 
we have 
\begin{align}
\label{eq:smooth-density-valid map-lower-bound}
\validminimaxrisk\bigl(\targetclass^s_{p',q}(L;m)\bigr)
\gtrsim
\begin{cases}
 n^{-\frac{2(1+s)}{d+2s}}, & d\geq2,\\[.5em]
 n^{-1}, & d=1.
\end{cases}
\end{align}
\end{corollary}
\noindent We note that the same lower bound holds for every restricted subclass of deterministic estimators.
Under related but stronger smoothness assumptions, the recent work of \citet{mena2025statistical} (see their Theorem 8) showed that the rectified flow map can be estimated 
at the slower rate of $n^{-2s/(2s + d - 1)}$. 
Corollary~\ref{cor:smooth-density-valid map-lower-bound} 
gives a complementary lower bound for any estimator whose output is evaluated through the valid map loss. Since the loss to any specified population transport map, including the rectified flow map when it is well-defined, dominates the valid map loss, the same lower bound applies to such target-specific map estimation problems as well.
This result highlights that in the standard non-parametric setups that have been studied in the OT map estimation literature, the curse of dimensionality is unavoidable, and targeting any (even sub-optimal) transport map does not alleviate these challenges.

\subsubsection*{Relation to stability assumptions for OT map estimation}

Past work summarized in Lemma~\ref{lem:ot-stability} has shown that, under smoothness and strong convexity of the Brenier potential, estimating the OT map can be reduced to estimating the underlying distributions in Wasserstein distance. The following proposition summarizes the implications of Lemmas~\ref{lem:ot-stability} and~\ref{lem:main}.

\begin{proposition}[Stability sandwiches the three risks]
\label{prop:stability-sandwich}
Suppose that $\mu$ is absolutely continuous and that, for every $\nu\in\targetclass$, the quadratic OT map $T_\nu$ from $\mu$ to $\nu$ exists. Assume further that there is a constant $C_{\rm stab}<\infty$ such that, for every distribution estimator $\widehat\nu$ under consideration, the plug-in OT map $T_{\widehat\nu}$ satisfies
\begin{align}
\label{eq:abstract-stability-assumption}
\int \|T_{\widehat\nu}(x)-T_\nu(x)\|_2^2\,d\mu(x)
\leq
C_{\rm stab} W_2^2(\widehat\nu,\nu)
\end{align}
for all $\nu\in\targetclass$. Then
\begin{align}
\label{eq:stability-sandwich}
\wtwominimaxrisk(\targetclass)
\leq
\validminimaxrisk(\targetclass)
\leq
\otminimaxrisk(\targetclass)
\leq
C_{\rm stab}\wtwominimaxrisk(\targetclass).
\end{align}
\end{proposition}

\noindent This proposition shows that in any model where a stability inequality of the form \eqref{eq:abstract-stability-assumption} holds and the Wasserstein distribution-estimation rate is known, valid map estimation, OT-map estimation, and Wasserstein distribution estimation all have the same minimax rate up to the stability constant. This result summarizes the formal sense in which targeting a sub-optimal transport map cannot yield any significant statistical benefit over targeting the OT map in stable regimes. 

\subsection{Separations between Valid Transport and OT}
\label{sec:separation}

The preceding results show that, under sufficiently strong stability assumptions, estimating a valid transport map is no easier than estimating the optimal transport map. We now show that this conclusion can fail once stability is weakened. The key insight is not that valid transport maps avoid the need to estimate the target distribution: Lemma~\ref{lem:main} shows that they do not. Rather, the key insight is that the OT map can be unstable to small perturbations of the target distribution, and in these cases the OT map risk can be substantially larger than the valid map risk.

\subsubsection{Polynomial Separation with a Benign Source}
We follow a classical scheme in proving minimax lower bounds of reducing from estimation to testing.
The following two-point reduction forms the basis for our minimax risk comparisons.
\begin{lemma}[Two-point reduction]
\label{lem:twopoint-separation}
Let $\targetclass=\{\nu_0,\nu_1\}$, let $T_i:=T_{\nu_i}$ be the corresponding OT maps, and set
\[
\Delta_{\OT}^2:=\|T_0-T_1\|_{L^2(\mu)}^2.
\]
Let
\[
e_n:=\inf_\psi \max_{i\in\{0,1\}}
\nu_i^{\otimes n}(\psi\neq i)
\]
be the minimax testing error for distinguishing $\nu_0$ from $\nu_1$. Then
\begin{align}
\label{eq:twopoint-ot-lower}
\otminimaxrisk(\targetclass)
\geq
\frac{1}{4}\Delta_{\OT}^2 e_n.
\end{align}
Moreover, for any maps $S_0,S_1$ satisfying $S_{i\#}\mu=\nu_i$,
\begin{align}
\label{eq:twopoint-valid-upper}
\validminimaxrisk(\targetclass)
\leq
\|S_0-S_1\|_{L^2(\mu)}^2.
\end{align}
In particular,
\[
\validminimaxrisk(\targetclass)
\leq
\inf_{S_{0\#}\mu=\nu_0,\,S_{1\#}\mu=\nu_1}
\|S_0-S_1\|_{L^2(\mu)}^2.
\]
\end{lemma}
\noindent To show a risk separation we need three ingredients: the targets must be statistically hard to distinguish, their OT maps must be far apart, and the two target distributions must admit valid transports from $\mu$ that are close as maps.

Our first construction is inspired by a well-known construction showing the H\"{o}lder irregularity of the OT map (see, for instance, Section 2.4 in the paper~\cite{letrouit2025lectures}). However, some careful modifications are needed to ensure that the resulting construction yields a valid minimax lower bound. 
Let $D\subset\mathbb R^2$ be the unit disk and let $\mu$ be the uniform distribution on $D$. Fix constants $a > 1$ and, for $\theta\in[0,\pi/4]$, write
\[
u_\theta=(\cos\theta,\sin\theta),
\qquad
T_\theta(x):= x+a\,\sgn(x\cdot u_\theta)u_\theta,
\qquad
\nu_\theta:=(T_\theta)_\#\mu.
\]
The map $T_\theta$ is the gradient, $\mu$-a.e., of the convex potential
\[
\varphi_\theta(x):=\frac{1}{2}\|x\|_2^2+a|x\cdot u_\theta|.
\]
Consequently $T_\theta$ is the OT map from $\mu$ to $\nu_\theta$. Geometrically, $T_\theta$ sends the two half-disks cut out by the line $x\cdot u_\theta=0$ to two separated half-disks of radius $1$, centered at $+ a u_\theta$ and $- a u_\theta$. This construction is illustrated in Figure~\ref{fig:rotating-halves}. The following lemma analyzes this construction: 
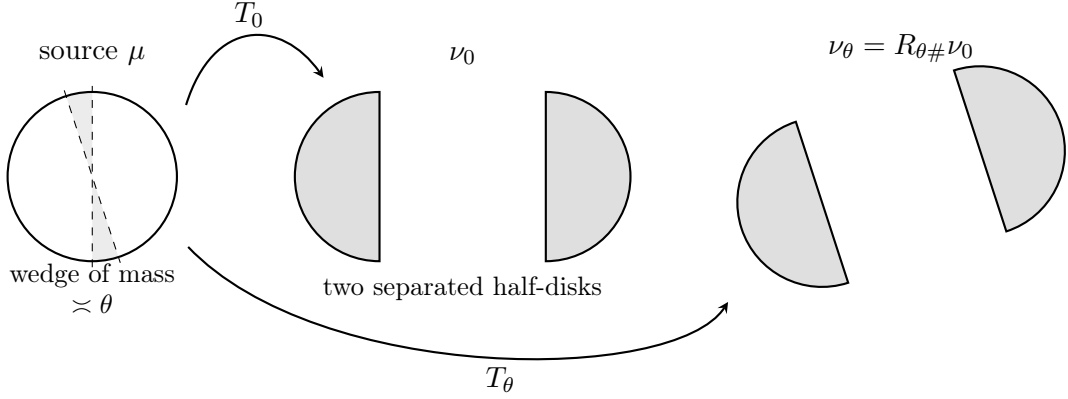
\begin{figure}[t]
\centering
\begin{tikzpicture}[scale=1,>=stealth]
  \def\rhalf{1.12}
  \def\sep{1.10}
  \def\DrawTargetHalves{%
    \path[fill=gray!25,draw=black,thick]
      (\sep,-\rhalf) arc[start angle=-90,end angle=90,radius=\rhalf] -- cycle;
    \path[fill=gray!25,draw=black,thick]
      (-\sep,\rhalf) arc[start angle=90,end angle=270,radius=\rhalf] -- cycle;
  }

  \begin{scope}[shift={(0,0)}]
    \node at (0,1.60) {source $\mu$};
    \fill[gray!15] (90:\rhalf) arc[start angle=90,end angle=108,radius=\rhalf] -- (0,0) -- cycle;
    \fill[gray!15] (270:\rhalf) arc[start angle=270,end angle=288,radius=\rhalf] -- (0,0) -- cycle;
    \draw[thick] (0,0) circle (\rhalf);
    \draw[dashed] (0,-1.20) -- (0,1.20);
    \draw[dashed] (108:1.20) -- (288:1.20);
    \node[align=center,font=\small] at (0,-1.48) {wedge of mass\\$\asymp\theta$};
  \end{scope}

  \draw[->,thick,shorten <=2pt,shorten >=2pt]
    (1.22,0.88) .. controls (1.60,2.10) and (2.50,2.10) ..
    node[pos=.50,above] {$T_0$} (3.10,1.25);

  \draw[->,thick,shorten <=2pt,shorten >=2pt]
    (1.22,-0.88) .. controls (3.00,-2.75) and (7.70,-2.75) ..
    node[pos=.52,below] {$T_\theta$} (8.45,-1.60);

  \begin{scope}[shift={(4.90,0)}]
    \node at (0,1.60) {$\nu_0$};
    \DrawTargetHalves
    \node[font=\small] at (0,-1.48) {two separated half-disks};
  \end{scope}

  \node at (10.70,1.70) {$\nu_\theta=R_{\theta\#}\nu_0$};
  \begin{scope}[shift={(10.70,0)},rotate=18]
    \DrawTargetHalves
  \end{scope}
\end{tikzpicture}
\caption{The rotating two-component construction. The target distributions differ by a small rotation, so their squared Wasserstein distance is $\lesssim \theta^2$. However, the sign boundary in the source rotates, and a wedge of source mass of order $\theta$ is sent to the opposite component. This ensures that the squared $L^2(\mu)$ distance between the OT maps is of order $\theta$.}
\label{fig:rotating-halves}
\end{figure}
\begin{lemma}
\label{lem:rotating-geometry}
There are constants $0<c<C<\infty$, depending only on $a$, such that for all sufficiently small $\theta>0$,
\begin{align}
\label{eq:rotating-map-gap}
\|T_\theta-T_0\|_{L^2(\mu)}^2 \asymp \theta,
\end{align}
while
\begin{align}
\label{eq:rotating-w2-gap}
W_2^2(\nu_\theta,\nu_0)
\lesssim
\theta^2,
\end{align}
and
\begin{align}
\label{eq:rotating-hellinger}
H^2(\nu_\theta,\nu_0)
\lesssim
\theta.
\end{align}
\end{lemma}

\noindent Lemmas~\ref{lem:twopoint-separation} and~\ref{lem:rotating-geometry} together yield the following result:
\begin{theorem}[Polynomial separation with a benign source]
\label{thm:john-polynomial-separation}
Let $\mu$ be the uniform distribution on the unit disk in $\mathbb R^2$. There exist constants $c,C>0$ such that for every sufficiently large $n$ there is a two-point target class
\[
\targetclass_n=\{\nu_0,\nu_{\theta_n}\},
\qquad
\theta_n=\kappa/n,
\]
with $\kappa>0$ sufficiently small, for which
\begin{align}
\label{eq:john-ot-lower}
\otminimaxrisk(\targetclass_n)
\gtrsim
\frac{1}{n},
\end{align}
and
\begin{align}
\label{eq:john-valid-upper}
\validminimaxrisk(\targetclass_n)
\lesssim
\frac{1}{n^2}.
\end{align}
In particular,
\[
\frac{\otminimaxrisk(\targetclass_n)}
{\validminimaxrisk(\targetclass_n)}
\gtrsim n .
\]
\end{theorem}

\noindent The source distribution $\mu$ in Theorem~\ref{thm:john-polynomial-separation} is quite benign: it is bounded above and below on a convex domain, and we believe with a bit more 
analytic effort the source could be replaced by a standard Normal distribution.
To place this result in context, a line of work~\cite{letrouit2024gluing,merigot2020,delalande2023,letrouit2025lectures} has shown that for source measures which are upper and lower bounded on a so-called ``John domain'' (which includes convex domains as a special case), OT maps to (arbitrary, compactly supported) targets $\nu, \eta$ satisfy stability bounds of the form:
\[
\|T_\nu-T_\eta\|_{L^2(\mu)}
\lesssim W_2(\nu,\eta)^{1/6}.
\]
This stability rules out the most extreme form of instability (exhibited in Theorem~\ref{thm:exponential-separation}) for fixed regular sources: the OT risk cannot remain bounded away from zero while the valid map risk becomes exponentially small. Nevertheless, the preceding construction shows that polynomial amplifications of the valid map risk are already possible for benign source distributions.

\subsubsection{An Exponential Separation for Adversarial Sources}

The previous example uses a fixed, regular source. It gives a polynomial separation, but the H\"older stability of OT maps for regular sources prevents the most extreme behavior. If we allow the source distribution to be more adversarially chosen a much larger separation is possible.

Once again we follow a classical scheme to prove a minimax lower bound (see Theorem 2.15 of \citet{tsybakov2008}) of constructing a collection of distributions, indexed by a hypercube,
where adjacent distributions on the hypercube are difficult to distinguish.
\begin{lemma}[Assouad lower bound]
\label{lem:assouad-map-hellinger}
Let $\{P_\sigma:\sigma\in\{\pm1\}^M\}$ be a family of distributions, and let
\[
T_\sigma:\mathcal X\to\mathbb R^d,
\qquad \sigma\in\{\pm1\}^M,
\]
be a corresponding family of maps. Let $\mu$ be a probability measure on
$\mathcal X$, and let $G_1,\ldots,G_M\subseteq\mathcal X$ be disjoint
measurable sets.

For $\sigma\in\{\pm1\}^M$, let $\sigma^{(j)}$ denote the vector obtained
from $\sigma$ by flipping the $j$th coordinate. Suppose that for each
$j=1,\ldots,M$ there exists $\Delta_j>0$ such that
\begin{align}
\label{eq:assouad-local-separation}
\int_{G_j}
\|T_\sigma(x)-T_{\sigma^{(j)}}(x)\|_2^2\,d\mu(x)
\geq
\Delta_j
\end{align}
for every $\sigma\in\{\pm1\}^M$. Suppose also that the Hellinger
affinities satisfy
\begin{align}
\label{eq:assouad-affinity}
\rho(P_\sigma^{\otimes n},P_{\sigma^{(j)}}^{\otimes n})
:=
\int \sqrt{dP_\sigma^{\otimes n} dP^{\otimes n}_{\sigma^{(j)}}}
\geq
\rho_0
\end{align}
for every $\sigma$ and every $j$, where $\rho_0\in(0,1]$.

Then
\begin{align}
\label{eq:assouad-map-conclusion}
\inf_{\widehat T}
\sup_{\sigma\in\{\pm1\}^M}
\mathbb E_\sigma
\int
\|\widehat T(x)-T_\sigma(x)\|_2^2\,d\mu(x)
\geq
\frac{1-\sqrt{1-\rho_0^2}}{8}
\sum_{j=1}^M \Delta_j .
\end{align}
Here the infimum is over all estimators $\widehat T$ based on $n$ 
observations from $P_\sigma$.
\end{lemma}

Our construction is inspired by another classical example (see, for instance, Section 1.1 in the paper~\cite{letrouit2025lectures} or Figure 2 in~\cite{letrouit2026unstable}) where the OT map is non-unique and consequently unstable to perturbations in the Wasserstein sense. However, once again several careful adjustments are needed to this basic construction to prove a statistical minimax lower bound. 

Our construction is illustrated in Figure~\ref{fig:exp-separation-construction}. Let  $\varepsilon=1/n$. Fix a constant $a>0$ and set $\theta_n = \exp(-n)$. We first describe the source distribution. 
Choose points $u_1<\cdots<u_n$, where $u_j = 4\varepsilon j$. 
Associated with each $u_j$ define two source centers (which we refer to as a block)
\[
x_{j,L}=(u_j-\varepsilon,0),
\qquad
x_{j,R}=(u_j+\varepsilon,0).
\]
Choose a radius $r_n>0$ satisfying
\[
r_n \ll  \varepsilon \theta_n .
\]
Let $B_{j,L}$ and $B_{j,R}$ be the Euclidean balls of radius $r_n$
centered at $x_{j,L}$ and $x_{j,R}$. Define $\mu$ to be the probability
distribution which assigns mass $\varepsilon/2$ uniformly to each of these $2n$
balls.

We now define the collection of target distributions. For each sign vector $\sigma=(\sigma_1,\ldots,\sigma_n)\in\{\pm1\}^n$,
define two target centers (which we also refer to as a block) associated with the center $u_j$ by
\[
y_{j,R}^{\sigma}=(u_j+\theta_n,\sigma_j a),
\qquad
y_{j,L}^{\sigma}=(u_j-\theta_n,-\sigma_j a),
\]
and let $\nu_\sigma$ assign mass $\varepsilon/2$ uniformly to each of the balls of
radius $r_n$ centered at $y_{j,L}^{\sigma}$ and $y_{j,R}^{\sigma}$, over
all $j=1,\ldots,n$. Let
\[
\targetclass_n:=\{\nu_\sigma:\sigma\in\{\pm1\}^n\}.
\]
With the source distribution and target class in place, Lemma~\ref{lem:lower-bound-analysis} in the Appendix analyzes this construction to control the desired transport map separation and Hellinger distances. Lemma~\ref{lem:lower-bound-analysis} and Lemma~\ref{lem:assouad-map-hellinger} together yield the following conclusion:
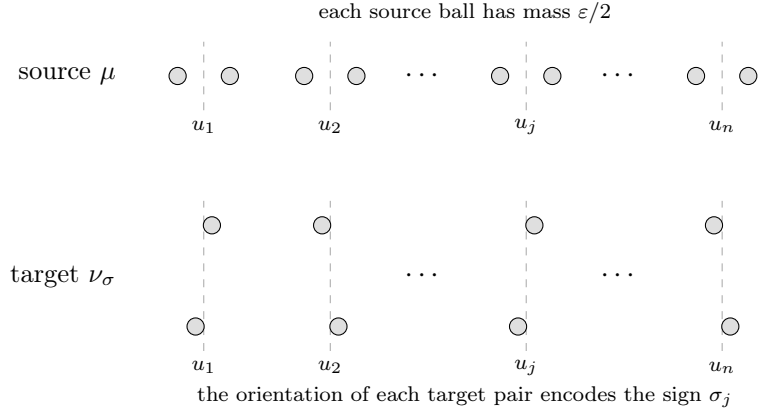
\begin{figure}[t]
\centering
\begin{tikzpicture}[scale=1.08,>=stealth]

  \def\r{0.105}
  \def\epsv{0.32}
  \def\thetav{0.10}
  \def\av{0.62}
  \def\ys{1.25}
  \def\yt{-1.20}

  \def\Sourceblock#1#2{%
    \begin{scope}[shift={(#1,\ys)}]
      \draw[dashed,gray!65] (0,-0.42) -- (0,0.42);
      \filldraw[fill=gray!25,draw=black] (-\epsv,0) circle (\r);
      \filldraw[fill=gray!25,draw=black] ( \epsv,0) circle (\r);
      \node[font=\scriptsize] at (0,-0.62) {$u_{#2}$};
    \end{scope}
  }

  \def\TargetPlus#1#2{%
    \begin{scope}[shift={(#1,\yt)}]
      \draw[dashed,gray!65] (0,-0.92) -- (0,0.92);
      \filldraw[fill=gray!25,draw=black] (-\thetav,-\av) circle (\r);
      \filldraw[fill=gray!25,draw=black] ( \thetav, \av) circle (\r);
      \node[font=\scriptsize] at (0,-1.12) {$u_{#2}$};
    \end{scope}
  }

  \def\TargetMinus#1#2{%
    \begin{scope}[shift={(#1,\yt)}]
      \draw[dashed,gray!65] (0,-0.92) -- (0,0.92);
      \filldraw[fill=gray!25,draw=black] (-\thetav, \av) circle (\r);
      \filldraw[fill=gray!25,draw=black] ( \thetav,-\av) circle (\r);
      \node[font=\scriptsize] at (0,-1.12) {$u_{#2}$};
    \end{scope}
  }

  \node[anchor=east,font=\small] at (-0.95,\ys) {source $\mu$};
  \node[anchor=east,font=\small] at (-0.95,\yt) {target $\nu_\sigma$};

  \Sourceblock{0.0}{1}
  \Sourceblock{1.55}{2}
  \node at (2.70,\ys) {$\cdots$};
  \Sourceblock{3.95}{j}
  \node at (5.10,\ys) {$\cdots$};
  \Sourceblock{6.35}{n}

  \TargetPlus{0.0}{1}
  \TargetMinus{1.55}{2}
  \node at (2.70,\yt) {$\cdots$};
  \TargetPlus{3.95}{j}
  \node at (5.10,\yt) {$\cdots$};
  \TargetMinus{6.35}{n}

  \node[font=\scriptsize,align=center] at (3.20,2.02)
  {each source ball has mass $\varepsilon/2$};

  \node[font=\scriptsize,align=center] at (3.20,-2.68)
  {the orientation of each target pair encodes the sign $\sigma_j$};

\end{tikzpicture}
\caption{Illustration of the exponential-separation construction. The source distribution $\mu$ is supported on $n$ well-separated blocks, with block $j$ consisting of two small balls centered at $(u_j-\varepsilon,0)$ and $(u_j+\varepsilon,0)$. For each sign vector $\sigma\in\{\pm1\}^n$, the target distribution $\nu_\sigma$ places two small balls near $(u_j-\theta_n,-\sigma_j a)$ and $(u_j+\theta_n,\sigma_j a)$. Each block encodes one hidden bit $\sigma_j$.}
\label{fig:exp-separation-construction}
\end{figure}

\begin{theorem}[Exponential separation]
\label{thm:exponential-separation}
For every sufficiently large $n$ the source distribution $\mu$ and target class $\targetclass_n$ described above ensure that:
\begin{align}
\label{eq:exp-ot-lower}
\otminimaxrisk(\targetclass_n)
\gtrsim 1,
\end{align}
while
\begin{align}
\label{eq:exp-valid-upper}
\validminimaxrisk(\targetclass_n)
\lesssim
\exp(-2n).
\end{align}
Consequently,
\[
\frac{\otminimaxrisk(\targetclass_n)}
{\validminimaxrisk(\targetclass_n)}
\gtrsim \exp(2n).
\]
\end{theorem}

The construction in Theorem~\ref{thm:exponential-separation} is intentionally adversarial. The source distribution changes with $n$, consists of many tiny components, and the radius $r_n$ must be much smaller than the horizontal perturbation scale $\theta_n=\exp(-n)$. Thus the source geometry and density become increasingly ill-conditioned. This is exactly what allows an exponential valid map improvement without contradicting qualitative or quantitative stability results for fixed regular sources.

The two examples we have given identify two different regimes. For a fixed regular source on a convex domain, OT stability rules out catastrophic instability, but it still permits polynomial separations between the OT map risk and valid map risk. With a more adversarial source distribution the OT map can be forced to encode many nearly invisible bits and OT map estimation can have constant risk while a fixed valid transport map is exponentially accurate.

\section{Discussion}

In practice, using diffusion and flow models involves a range of design decisions that can significantly influence the quality of generated samples. These include choices about neural network architectures, discretization schemes, noise schedules, and optimizer hyperparameters. Ultimately, these choices can serve to inject bias and regularize the methods in favorable ways, pulling them away from the idealized goal of approximating an unregularized transport map or plan. Given some of the results of this paper on the hardness of approximating even a valid transport map in some settings, it seems likely that explaining the statistical benefits of these methods in these challenging settings will require a deeper understanding of these practical design choices.

To distinguish optimal transport from sub-optimal transport, we considered a pair of idealized constructions. These examples are useful because they isolate a specific mechanism: the optimal transport map may encode a fragile global matching structure, while many non-optimal transports can avoid this matching and still push the source distribution to the target. A difficult but natural question is to understand the extent to which real data distributions fall into stable or unstable regimes of this kind. Stability is not simply a property of the source distribution or the target distribution in isolation; it depends on the interaction between the two, and on the geometry of their supports.
This issue seems particularly relevant for modern generative modeling. Classical regularity theory gives strong stability guarantees under convexity, density, and curvature assumptions and recent quantitative stability results show that even fairly weak geometric assumptions can imply polynomial control of the OT map \cite{letrouit2024gluing,merigot2020,delalande2023,letrouit2025lectures}. On the other hand, these assumptions are far from the geometry of realistic data distributions, which often have multiple modes, thin or lower-dimensional structure, and unevenly separated components. A useful direction for future work would be to develop data-driven diagnostics that distinguish stable from unstable transport problems. Such a diagnostic could distinguish when the optimal transport map is a statistically reasonable target, from situations when sub-optimal transport mechanisms are not merely computationally convenient but statistically preferable.

We focused on the ability of various methods to estimate valid transport maps. This provides a principled framework for comparison, even when methods target different population quantities. However, one should be cautious when interpreting the practical implications of our results: in generative modeling, the objectives of sample quality and diversity are often only weakly correlated with the goal of estimating a transport map. Developing alternative formulations for these practical goals that still facilitate rigorous mathematical analysis is an interesting direction for future work.

Finally, our framework also seems to lend itself naturally to the creation of \emph{computational} separations. We focused here on statistical separations: regimes where the tasks of estimating optimal and sub-optimal transport maps have different sample complexities. A complementary question is whether there are settings where
finding a valid sub-optimal transport map can be rigorously shown to be computationally easier than finding the optimal one. 

\section*{Acknowledgements}
We are grateful to Narayanaswamy Balakrishnan, Arun Kuchibhotla, Gonzalo Mena, Andrej Risteski, Dejan Slepcev and Larry Wasserman for helpful discussions. We used ChatGPT to improve, proofread and produce figures for this manuscript. This research was supported in part by the NSF grant DMS-2310632.

\bibliographystyle{abbrvnat}
\bibliography{transport}

\begin{thebibliography}{29}
\providecommand{\natexlab}[1]{#1}
\providecommand{\url}[1]{\texttt{#1}}
\expandafter\ifx\csname urlstyle\endcsname\relax
  \providecommand{\doi}[1]{doi: #1}\else
  \providecommand{\doi}{doi: \begingroup \urlstyle{rm}\Url}\fi

\bibitem[Albergo et~al.(2025)Albergo, Boffi, and
  Vanden-Eijnden]{albergo2023stochastic}
M.~S. Albergo, N.~M. Boffi, and E.~Vanden-Eijnden.
\newblock Stochastic interpolants: A unifying framework for flows and
  diffusions.
\newblock \emph{Journal of Machine Learning Research}, 26\penalty0
  (209):\penalty0 1--80, 2025.

\bibitem[Balakrishnan and Manole(2026)]{balakrishnan2026stability}
S.~Balakrishnan and T.~Manole.
\newblock Stability bounds for smooth optimal transport maps and their
  statistical implications.
\newblock \emph{Electronic Journal of Statistics}, 2026.
\newblock To appear.

\bibitem[Balakrishnan et~al.(2025)Balakrishnan, Manole, and
  Wasserman]{balakrishnan2025statistical}
S.~Balakrishnan, T.~Manole, and L.~Wasserman.
\newblock Statistical inference for optimal transport maps: Recent advances and
  perspectives.
\newblock \emph{Statistical Science}, 2025.

\bibitem[Brenier(1991)]{brenier1991}
Y.~Brenier.
\newblock Polar factorization and monotone rearrangement of vector-valued
  functions.
\newblock \emph{Communications on Pure and Applied Mathematics}, 44:\penalty0
  375--417, 1991.

\bibitem[Caffarelli(1992)]{caffarelli1992regularity}
L.~A. Caffarelli.
\newblock The {{Regularity}} of {{Mappings}} with a {{Convex Potential}}.
\newblock \emph{Journal of the American Mathematical Society}, 5:\penalty0
  99--104, 1992.

\bibitem[Deb et~al.(2021)Deb, Ghosal, and Sen]{deb2021}
N.~Deb, P.~Ghosal, and B.~Sen.
\newblock Rates of {{Estimation}} of {{Optimal Transport Maps}} using
  {{Plug}}-in {{Estimators}} via {{Barycentric Projections}}.
\newblock \emph{Advances in Neural Information Processing Systems 34}, 2021.

\bibitem[Delalande and Merigot(2023)]{delalande2023}
A.~Delalande and Q.~Merigot.
\newblock Quantitative stability of optimal transport maps under variations of
  the target measure.
\newblock \emph{Duke Mathematical Journal}, 172\penalty0 (17):\penalty0
  3321--3357, 2023.

\bibitem[Dinh et~al.(2017)Dinh, Sohl-Dickstein, and Bengio]{dinh2017density}
L.~Dinh, J.~Sohl-Dickstein, and S.~Bengio.
\newblock Density estimation using real {NVP}.
\newblock In \emph{International Conference on Learning Representations}, 2017.

\bibitem[Divol et~al.(2022)Divol, {Niles-Weed}, and Pooladian]{divol2022a}
V.~Divol, J.~{Niles-Weed}, and A.-A. Pooladian.
\newblock Optimal transport map estimation in general function spaces.
\newblock \emph{The Annals of Statistics}, 2022.

\bibitem[Figalli(2017)]{figalli2017monge}
A.~Figalli.
\newblock \emph{The {Monge--Amp{\`e}re} Equation and Its Applications}.
\newblock Zurich Lectures in Advanced Mathematics. European Mathematical
  Society, Z{\"u}rich, 2017.

\bibitem[Ho et~al.(2020)Ho, Jain, and Abbeel]{ho2020denoising}
J.~Ho, A.~Jain, and P.~Abbeel.
\newblock Denoising diffusion probabilistic models.
\newblock In \emph{Advances in Neural Information Processing Systems},
  volume~33, pages 6840--6851. Curran Associates, Inc., 2020.

\bibitem[H{\"u}tter and Rigollet(2021)]{hutter2021}
J.-C. H{\"u}tter and P.~Rigollet.
\newblock Minimax rates of estimation for smooth optimal transport maps.
\newblock \emph{The Annals of Statistics}, 49:\penalty0 1166--1194, 2021.

\bibitem[Letrouit(2025)]{letrouit2025lectures}
C.~Letrouit.
\newblock Lectures on quantitative stability of optimal transport, May 2025.
\newblock Notes du Cours Peccot, Coll{\`e}ge de France, May--June 2025.
  Preliminary version.

\bibitem[Letrouit(2026)]{letrouit2026unstable}
C.~Letrouit.
\newblock Unstable optimal transport maps.
\newblock \emph{Comptes Rendus. Math{\'e}matique}, 364:\penalty0 333--344,
  2026.

\bibitem[Letrouit and M{\'e}rigot(2024)]{letrouit2024gluing}
C.~Letrouit and Q.~M{\'e}rigot.
\newblock Gluing methods for quantitative stability of optimal transport maps,
  2024.

\bibitem[Lipman et~al.(2023)Lipman, Chen, Ben-Hamu, Nickel, and
  Le]{lipman2023flowmatching}
Y.~Lipman, R.~T.~Q. Chen, H.~Ben-Hamu, M.~Nickel, and M.~Le.
\newblock Flow matching for generative modeling.
\newblock In \emph{The Eleventh International Conference on Learning
  Representations}, 2023.

\bibitem[Liu et~al.(2023)Liu, Gong, and Liu]{liu2022rectified}
X.~Liu, C.~Gong, and Q.~Liu.
\newblock Flow straight and fast: Learning to generate and transfer data with
  rectified flow.
\newblock In \emph{The Eleventh International Conference on Learning
  Representations}, 2023.

\bibitem[Manole et~al.(2024)Manole, Balakrishnan, Niles-Weed, and
  Wasserman]{manole2024plugin}
T.~Manole, S.~Balakrishnan, J.~Niles-Weed, and L.~Wasserman.
\newblock Plugin estimation of smooth optimal transport maps.
\newblock \emph{The Annals of Statistics}, 52\penalty0 (3):\penalty0 966--998,
  2024.

\bibitem[Mena et~al.(2025)Mena, Kuchibhotla, and
  Wasserman]{mena2025statistical}
G.~Mena, A.~K. Kuchibhotla, and L.~Wasserman.
\newblock Statistical properties of {Rectified Flow}.
\newblock arXiv.2511.03193, 2025.

\bibitem[M{\'e}rigot et~al.(2020)M{\'e}rigot, Delalande, and
  Chazal]{merigot2020}
Q.~M{\'e}rigot, A.~Delalande, and F.~Chazal.
\newblock Quantitative stability of optimal transport maps and linearization of
  the 2-{W}asserstein space.
\newblock In \emph{International Conference on Artificial Intelligence and
  Statistics}, pages 3186--3196. PMLR, 2020.

\bibitem[Niles-Weed and Berthet(2022)]{weed2019a}
J.~Niles-Weed and Q.~Berthet.
\newblock Minimax estimation of smooth densities in wasserstein distance.
\newblock \emph{The Annals of Statistics}, 50\penalty0 (3):\penalty0
  1519--1540, 2022.

\bibitem[Papamakarios et~al.(2021)Papamakarios, Nalisnick, Rezende, Mohamed,
  and Lakshminarayanan]{papamakarios2021normalizing}
G.~Papamakarios, E.~Nalisnick, D.~J. Rezende, S.~Mohamed, and
  B.~Lakshminarayanan.
\newblock Normalizing flows for probabilistic modeling and inference.
\newblock \emph{Journal of Machine Learning Research}, 22\penalty0
  (57):\penalty0 1--64, 2021.

\bibitem[Peyr{\'e} and Cuturi(2019)]{peyre2019computational}
G.~Peyr{\'e} and M.~Cuturi.
\newblock Computational optimal transport with applications to data sciences.
\newblock \emph{Foundations and Trends in Machine Learning}, 11\penalty0
  (5-6):\penalty0 355--607, 2019.

\bibitem[Rezende and Mohamed(2015)]{rezende2015variational}
D.~Rezende and S.~Mohamed.
\newblock Variational inference with normalizing flows.
\newblock In \emph{Proceedings of the 32nd International Conference on Machine
  Learning}, volume~37 of \emph{Proceedings of Machine Learning Research},
  pages 1530--1538. PMLR, 2015.

\bibitem[Sohl-Dickstein et~al.(2015)Sohl-Dickstein, Weiss, Maheswaranathan, and
  Ganguli]{sohl2015deep}
J.~Sohl-Dickstein, E.~Weiss, N.~Maheswaranathan, and S.~Ganguli.
\newblock Deep unsupervised learning using nonequilibrium thermodynamics.
\newblock In \emph{Proceedings of the 32nd International Conference on Machine
  Learning}, volume~37 of \emph{Proceedings of Machine Learning Research},
  pages 2256--2265. PMLR, 2015.

\bibitem[Song et~al.(2021)Song, Sohl-Dickstein, Kingma, Kumar, Ermon, and
  Poole]{song2021score}
Y.~Song, J.~Sohl-Dickstein, D.~P. Kingma, A.~Kumar, S.~Ermon, and B.~Poole.
\newblock Score-based generative modeling through stochastic differential
  equations.
\newblock In \emph{International Conference on Learning Representations}, 2021.

\bibitem[Tsybakov(2008)]{tsybakov2008}
A.~B. Tsybakov.
\newblock \emph{Introduction to Nonparametric Estimation}.
\newblock {Springer Science \& Business Media}, 2008.

\bibitem[Villani(2003)]{villani2003topics}
C.~Villani.
\newblock \emph{Topics in Optimal Transportation}, volume~58 of \emph{Graduate
  Studies in Mathematics}.
\newblock American Mathematical Society, Providence, RI, 2003.

\bibitem[Villani(2009)]{villani2009optimal}
C.~Villani.
\newblock \emph{Optimal Transport: Old and New}, volume 338 of
  \emph{Grundlehren der mathematischen Wissenschaften}.
\newblock Springer, Berlin, Heidelberg, 2009.

\end{thebibliography}

\appendix

\clearpage

\section{The One-Dimensional Case}
\label{sec:oned}
The one-dimensional case of transport is special and in this section we discuss it briefly. In higher dimensions, stability of the OT map is a nontrivial regularity property. In one dimension, however, monotonicity gives an exact stability identity.

Throughout this subsection assume that $d=1$ and that the source distribution $\mu$ is non-atomic. For any $\eta\in\mathcal{P}_2(\mathbb{R})$, let $F_\eta$ denote its distribution function and let
\[
F_\eta^{-1}(t) := \inf\{x\in\mathbb{R}:F_\eta(x)\geq t\},
\qquad 0<t<1,
\]
denote its generalized inverse. Define the monotone rearrangement from $\mu$ to $\eta$ by
\[
T_\eta(x) := F_\eta^{-1}(F_\mu(x)).
\]
Then $T_{\eta\#}\mu=\eta$, and $T_\eta$ is the one-dimensional optimal transport map from $\mu$ to $\eta$ for the squared Euclidean cost. In particular, if the true target is $\nu$, then the OT map is
\[
T_0 = T_\nu = F_\nu^{-1}\circ F_\mu.
\]

The following observation shows that in one dimension any estimator can be replaced by a monotone estimator whose OT-map error is exactly the Wasserstein error of the induced target distribution.

\begin{lemma}[Monotonization in one dimension]
\label{lem:one-d-monotonization}
Let $\widehat{T}:\mathbb{R}\to\mathbb{R}$ be any measurable estimator, and write
\[
\widehat{\nu}:=\widehat{T}_{\#}\mu.
\]
Define its monotone rearrangement by
\[
\widetilde{T}:=T_{\widehat{\nu}}
=
F_{\widehat{\nu}}^{-1}\circ F_\mu.
\]
Then
\[
\widetilde{T}_{\#}\mu=\widehat{\nu},
\]
and
\begin{align}
\label{eq:one-d-ot-isometry}
\int |\widetilde{T}(x)-T_0(x)|^2\,d\mu(x)
=
W_2^2(\widehat{\nu},\nu).
\end{align}
Moreover,
\begin{align}
\label{eq:one-d-valid-risk}
\inf_{T_{\#}\mu=\nu}
\int |\widetilde{T}(x)-T(x)|^2\,d\mu(x)
=
W_2^2(\widehat{\nu},\nu).
\end{align}
\end{lemma}

\begin{proof}
Since $\mu$ is non-atomic, if $X\sim\mu$, then $U:=F_\mu(X)$ is uniformly distributed on $(0,1)$. Therefore,
\[
\widetilde{T}(X)=F_{\widehat{\nu}}^{-1}(U),
\qquad
T_0(X)=F_\nu^{-1}(U).
\]
Using the one-dimensional quantile representation of the Wasserstein distance, we obtain
\[
\int |\widetilde{T}(x)-T_0(x)|^2\,d\mu(x)
=
\int_0^1
|F_{\widehat{\nu}}^{-1}(t)-F_\nu^{-1}(t)|^2\,dt
=
W_2^2(\widehat{\nu},\nu),
\]
which proves \eqref{eq:one-d-ot-isometry}.

For \eqref{eq:one-d-valid-risk}, the lower bound follows from Lemma~\ref{lem:main}, since $\widetilde{T}_{\#}\mu=\widehat{\nu}$. For the matching upper bound, simply choose the valid transport map $T=T_0$ in the infimum and use \eqref{eq:one-d-ot-isometry}. This proves the claim.
\end{proof}

This identity has a useful minimax consequence. Let $\targetclass\subseteq\mathcal{P}_2(\mathbb{R})$ be a class of target distributions, and recall the minimax Wasserstein distribution-estimation risk
\[
\wtwominimaxrisk(\targetclass)
:=
\inf_{\widehat{\nu}}
\sup_{\nu\in\targetclass}
\mathbb{E}_\nu W_2^2(\widehat{\nu},\nu),
\]
where the infimum is over all estimators of the target distribution based on the sample $Y_1,\ldots,Y_n$. Then
\begin{align}
\label{eq:one-d-minimax-equivalence}
\inf_{\widehat{T}}
\sup_{\nu\in\targetclass}
\mathbb{E}_\nu
\int |\widehat{T}(x)-T_\nu(x)|^2\,d\mu(x)
=
\inf_{\widehat{T}}
\sup_{\nu\in\targetclass}
\mathbb{E}_\nu
\inf_{T_{\#}\mu=\nu}
\int |\widehat{T}(x)-T(x)|^2\,d\mu(x)
=
\wtwominimaxrisk(\targetclass).
\end{align}
Indeed, the lower bounds follow because every map estimator $\widehat{T}$ induces a distribution estimator $\widehat{\nu}=\widehat{T}_{\#}\mu$, and both the OT loss and the valid map loss are bounded below by $W_2^2(\widehat{\nu},\nu)$. Conversely, given any distribution estimator $\widehat{\nu}$, the monotone estimator $T_{\widehat{\nu}}$ attains exactly this Wasserstein error by Lemma~\ref{lem:one-d-monotonization}.

In one dimension, estimating the OT map, estimating a valid transport map, and estimating the target distribution in squared Wasserstein distance are all the same minimax problem. In particular, there is no one-dimensional analogue of the separation phenomena we constructed in our work: any mechanism that makes the OT map hard to learn also makes the induced target distribution hard to estimate in $W_2$, and therefore also makes the problem of estimating a valid transport map hard.

\section{Stochastic Maps}
\label{app:stochastic}

While in the main paper we focused primarily on the setting with deterministic transport maps, analogous results hold more generally when studying stochastic maps. Methods that form
the basis for stochastic diffusion models inherently produce stochastic maps. We first give an analogue of Lemma~\ref{lem:main} for stochastic maps.

\begin{lemma}[Stochastic projection risk]
\label{lem:stochastic-main}
Let $\widehat{K}$ be a Markov kernel from $\mathbb{R}^d$ to $\mathbb{R}^d$, and define
\[
\widehat{\nu}(A)
=
\int \widehat{K}(A\mid x)\,d\mu(x),
\qquad A\subseteq \mathbb{R}^d.
\]
Then, for every $\nu\in\mathcal{P}_p(\mathbb{R}^d)$,
\begin{align}
\label{eq:stochastic-projection}
\inf_{K:\,\int K(\cdot\mid x)d\mu(x)=\nu}
\int W_p^p\bigl(\widehat{K}(\cdot\mid x),K(\cdot\mid x)\bigr)\,d\mu(x)
=
W_p^p(\widehat{\nu},\nu).
\end{align}
\end{lemma}

\begin{proof}
We first prove the lower bound. Fix any kernel $K$ such that
\[
\int K(\cdot\mid x)\,d\mu(x)=\nu.
\]
For each $x$, let $\gamma_x$ be an optimal coupling between $\widehat{K}(\cdot\mid x)$ and $K(\cdot\mid x)$. More formally, one may use measurable $\varepsilon$-optimal couplings and then let $\varepsilon\downarrow 0$. Integrating these couplings over $x$ gives a coupling $\gamma$ of $\widehat{\nu}$ and $\nu$. Therefore,
\[
W_p^p(\widehat{\nu},\nu)
\leq
\int \|y-\widehat{y}\|_2^p\,d\gamma(\widehat{y},y)
=
\int W_p^p\bigl(\widehat{K}(\cdot\mid x),K(\cdot\mid x)\bigr)\,d\mu(x).
\]
Taking the infimum over $K$ gives the lower bound.

For the upper bound, let $\lambda$ be an optimal coupling of $\widehat{\nu}$ and $\nu$. Disintegrate it as
\[
d\lambda(\widehat{y},y)
=
d\widehat{\nu}(\widehat{y})\,Q(dy\mid \widehat{y}).
\]
Now define a new kernel
\[
K^\star(dy\mid x)
=
\int Q(dy\mid \widehat{y})\,\widehat{K}(d\widehat{y}\mid x).
\]
Then
\[
\int K^\star(\cdot\mid x)\,d\mu(x)=\nu,
\]
so $K^\star$ is a valid stochastic transport from $\mu$ to $\nu$. Moreover, for each $x$, the joint law
\[
\widehat{K}(d\widehat{y}\mid x)\,Q(dy\mid \widehat{y})
\]
is a coupling of $\widehat{K}(\cdot\mid x)$ and $K^\star(\cdot\mid x)$. Hence
\begin{align*}
\int W_p^p\bigl(\widehat{K}(\cdot\mid x),K^\star(\cdot\mid x)\bigr)\,d\mu(x)
&\leq
\int \|\widehat{y}-y\|_2^p
\,\widehat{K}(d\widehat{y}\mid x)\,Q(dy\mid \widehat{y})\,d\mu(x) \\
&=
\int \|\widehat{y}-y\|_2^p\,d\lambda(\widehat{y},y) \\
&=
W_p^p(\widehat{\nu},\nu).
\end{align*}
Together with the lower bound, this proves \eqref{eq:stochastic-projection}.
\end{proof}

We can similarly derive a parallel result to that of Lemma~\ref{lem:wasserstein-lower-bound-valid map}. 
If $\widehat K$ is a Markov kernel from $\mathbb R^d$ to $\mathbb R^d$, define its induced terminal law by
\begin{align*}
\widehat\nu(A):=\int \widehat K(A\mid x)\,d\mu(x),
\qquad A\subseteq\mathbb R^d.
\end{align*}
We can define the valid loss in the stochastic setting as:
\begin{align*}
\lvalids(\widehat{K})
:=
\inf_{K:\,\int K(\cdot\mid x)d\mu(x)=\nu}
\int W_2^2\bigl(\widehat K(\cdot\mid x),K(\cdot\mid x)\bigr)\,d\mu(x).
\end{align*}
Let $\validsminimaxrisk(\targetclass)$ denote the corresponding minimax risk.

\begin{lemma}
\label{lem:kernel-risk-equals-w2-risk}
For any $\mu\in\mathcal P_2(\mathbb R^d)$ and any $\targetclass\subseteq\mathcal P_2(\mathbb R^d)$,
\begin{align}
\label{eq:kernel-risk-equals-w2-risk}
\validsminimaxrisk(\targetclass)
=
\wtwominimaxrisk(\targetclass).
\end{align}
\end{lemma}

\begin{proof}
We first prove the lower bound. Fix any stochastic estimator $\widehat K$, and let $\widehat\nu$ be its induced terminal law. For any valid kernel $K$ from $\mu$ to $\nu$, the collection of optimal couplings between $\widehat K(\cdot\mid x)$ and $K(\cdot\mid x)$, integrated over $x\sim\mu$, gives a coupling of $\widehat\nu$ and $\nu$. Therefore,
\begin{align*}
\int W_2^2\bigl(\widehat K(\cdot\mid x),K(\cdot\mid x)\bigr)\,d\mu(x)
\geq
W_2^2(\widehat\nu,\nu).
\end{align*}
Taking the infimum over valid $K$ gives
\[
\lvalids(\widehat K)
\geq
W_2^2(\widehat\nu,\nu),
\]
and the same minimax argument as in Lemma~\ref{lem:wasserstein-lower-bound-valid map} yields
\[
\validsminimaxrisk(\targetclass)
\geq
\wtwominimaxrisk(\targetclass).
\]

For the reverse inequality, fix any distribution estimator $\widetilde\nu$ and define the data-dependent kernel
\[
\widehat K(\cdot\mid x):=\widetilde\nu(\cdot),
\qquad x\in\mathbb R^d.
\]
This kernel ignores $x$ and has induced terminal law $\widetilde\nu$. For any target $\nu$, the kernel $K(\cdot\mid x):=\nu(\cdot)$ is valid from $\mu$ to $\nu$, and hence
\begin{align*}
\lvalids(\widehat K)
&\leq
\int W_2^2(\widetilde\nu,\nu)\,d\mu(x)
=
W_2^2(\widetilde\nu,\nu).
\end{align*}
Taking expectation, supremum over $\nu$, and infimum over $\widetilde\nu$ proves the reverse inequality.
\end{proof}

\section{Detailed Proofs}
\label{app:proofs}

\subsection{Proof of Lemma~\ref{lem:main}}

We first prove the lower bound. Fix any measurable map $T$ such that $T_{\#}\mu=\nu$, and let $X\sim\mu$. Then
$(\widehat{T}(X),T(X))$
is a coupling of $\widehat{\nu}$ and $\nu$. Therefore, by the definition of the Wasserstein distance,
\[
\int \|\widehat{T}(x)-T(x)\|_2^p\,d\mu(x)
=
\mathbb{E}\|\widehat{T}(X)-T(X)\|_2^p
\geq
W_p^p(\widehat{\nu},\nu).
\]
Taking the infimum over all $T$ satisfying $T_{\#}\mu=\nu$ gives \eqref{eq:valid-risk-lower}.

Now suppose that there exists an optimal Monge map $S$ from $\widehat{\nu}$ to $\nu$. Define
\[
T^\star := S\circ \widehat{T}.
\]
Then
\[
T^\star_{\#}\mu
=
S_{\#}(\widehat{T}_{\#}\mu)
=
S_{\#}\widehat{\nu}
=
\nu,
\]
so $T^\star$ is a valid transport map from $\mu$ to $\nu$. Moreover,
\[
\int \|\widehat{T}(x)-T^\star(x)\|_2^p\,d\mu(x)
=
\int \|\widehat{T}(x)-S(\widehat{T}(x))\|_2^p\,d\mu(x).
\]
Since $\widehat{T}_{\#}\mu=\widehat{\nu}$, the last display equals
\[
\int \|z-S(z)\|_2^p\,d\widehat{\nu}(z)
=
W_p^p(\widehat{\nu},\nu).
\]
Combining this upper bound with \eqref{eq:valid-risk-lower} proves equality.

\subsection{Proof of Lemma~\ref{lem:wasserstein-lower-bound-valid map}}
Fix any map estimator $\widehat T$, and write
\[
\widehat\nu:=\widehat T_\#\mu.
\]
Then $\widehat\nu$ is a distribution estimator of $\nu$. By Lemma~\ref{lem:main}, for every $\nu\in\targetclass$,
\begin{align*}
\inf_{T_\#\mu=\nu}
\int \|\widehat T(x)-T(x)\|_2^2\,d\mu(x)
\geq
W_2^2(\widehat\nu,\nu).
\end{align*}
Taking expectation, supremum over $\nu\in\targetclass$, and then infimum over $\widehat T$ gives
\begin{align*}
\validminimaxrisk(\targetclass)
&\geq
\inf_{\widehat T}
\sup_{\nu\in\targetclass}
\mathbb E_\nu W_2^2(\widehat T_\#\mu,\nu) \\
&\geq
\inf_{\widehat\nu}
\sup_{\nu\in\targetclass}
\mathbb E_\nu W_2^2(\widehat\nu,\nu),
\end{align*}
where the last infimum is over all distribution estimators, a larger class than those necessarily induced by map estimators. This proves the claim.

\subsection{Proof of Proposition~\ref{prop:stability-sandwich}}
The first inequality is Lemma~\ref{lem:wasserstein-lower-bound-valid map}. The second inequality follows pointwise from the fact that the OT map $T_\nu$ is one valid transport map from $\mu$ to $\nu$, so
\[
\inf_{T_\#\mu=\nu}
\int \|\widehat T(x)-T(x)\|_2^2\,d\mu(x)
\leq
\int \|\widehat T(x)-T_\nu(x)\|_2^2\,d\mu(x).
\]
For the last inequality, take any distribution estimator $\widehat\nu$ and output the plug-in OT map $T_{\widehat\nu}$. By \eqref{eq:abstract-stability-assumption}, its OT-map risk is at most $C_{\rm stab}\mathbb E_\nu W_2^2(\widehat\nu,\nu)$. Taking the supremum over $\nu\in\mathfrak N$ and then the infimum over $\widehat\nu$ gives the desired bound.

\subsection{Proof of Lemma~\ref{lem:twopoint-separation}}
This result follows from a standard reduction from estimation to testing (see~\cite{tsybakov2008}).
Fix an estimator $\widehat T$. It induces a test
\[
\psi(Y_1,\ldots,Y_n)
\in
\argmin_{i\in\{0,1\}}
\|\widehat T-T_i\|_{L^2(\mu)}.
\]
On the event $\{\psi\neq i\}$, the triangle inequality gives
\[
\Delta_{\OT}
=
\|T_0-T_1\|_{L^2(\mu)}
\leq
\|\widehat T-T_i\|_{L^2(\mu)}+
\|\widehat T-T_{1-i}\|_{L^2(\mu)}
\leq
2\|\widehat T-T_i\|_{L^2(\mu)}.
\]
Thus
\[
\|\widehat T-T_i\|_{L^2(\mu)}^2
\geq
\frac14\Delta_{\OT}^2\mathbf 1\{\psi\neq i\}.
\]
Taking expectations, then the maximum over $i$, and finally the infimum over $\widehat T$, proves \eqref{eq:twopoint-ot-lower}.

For the valid map upper bound, fix $S_0,S_1$ with $S_{i\#}\mu=\nu_i$ and use the estimator $\widehat T=S_0$, which ignores the data. Under $\nu_0$ its valid map loss is zero, while under $\nu_1$ its valid map loss is at most $\|S_0-S_1\|_{L^2(\mu)}^2$, since $S_1$ is an admissible valid map. This proves \eqref{eq:twopoint-valid-upper}.

\subsection{Proof of Lemma~\ref{lem:rotating-geometry}}
The identity \eqref{eq:rotating-map-gap} follows by direct calculation. Let $s_0(x)=\sgn(x\cdot u_0)$ and $s_\theta(x)=\sgn(x\cdot u_\theta)$. Since the linear terms cancel,
\[
T_\theta(x)-T_0(x)=a\{s_\theta(x)u_\theta-s_0(x)u_0\}.
\]
The signs $s_0$ and $s_\theta$ disagree on two wedges whose total $\mu$-mass is $\theta/\pi$. On the complement, the squared difference is $4a^2\sin^2(\theta/2)$; on the wedges, it is $4a^2\cos^2(\theta/2)$. Therefore
\[
\|T_\theta-T_0\|_{L^2(\mu)}^2
=
4a^2\left[
\sin^2(\theta/2)\left(1-\frac{\theta}{\pi}\right)
+
\cos^2(\theta/2)\frac{\theta}{\pi}
\right],
\]
which is comparable to $\theta$ for small $\theta$.

Next observe that $\nu_\theta=(R_\theta)_\#\nu_0$, where $R_\theta$ denotes rotation by angle $\theta$. Thus the coupling $Y\mapsto R_\theta Y$, with $Y\sim\nu_0$, gives
\[
W_2^2(\nu_0,\nu_\theta)
\leq
\mathbb E\|Y-R_\theta Y\|_2^2
\lesssim
\theta^2,
\]
because $\nu_0$ is supported in a fixed bounded set.

Finally, $\nu_0$ has density $(\pi)^{-1}\mathbf 1_{\Omega_0}$, where $\Omega_0$ is a finite union of half-disks, and $\nu_\theta$ has density $(\pi)^{-1}\mathbf 1_{R_\theta\Omega_0}$. Since $\Omega_0$ has finite perimeter and is bounded,
\[
|\Omega_0\triangle R_\theta\Omega_0|\lesssim\theta.
\]
Consequently $H^2(\nu_0,\nu_\theta)\lesssim \theta$.

\subsection{Proof of Theorem~\ref{thm:john-polynomial-separation}}
By Lemma~\ref{lem:rotating-geometry}, $H^2(\nu_0,\nu_{\theta_n})\lesssim\theta_n$. Choosing $\kappa>0$ small enough ensures that $nH^2(\nu_0,\nu_{\theta_n})$ is bounded by a sufficiently small universal constant. Hence the product measures $\nu_0^{\otimes n}$ and $\nu_{\theta_n}^{\otimes n}$ have total variation distance bounded away from one, and the testing error $e_n$ in Lemma~\ref{lem:twopoint-separation} is bounded below by a universal constant. Combining Lemma~\ref{lem:twopoint-separation} with \eqref{eq:rotating-map-gap} gives
\[
\otminimaxrisk(\targetclass_n)
\gtrsim
\theta_n
\gtrsim
\frac{1}{n}.
\]

For the valid map upper bound, use the estimator $\widehat T=T_0$, which ignores the data. The loss is zero under $\nu_0$. Under $\nu_{\theta_n}$, the map $R_{\theta_n}T_0$ is a valid transport from $\mu$ to $\nu_{\theta_n}$. Therefore
\[
\inf_{S_\#\mu=\nu_{\theta_n}}
\|T_0-S\|_{L^2(\mu)}^2
\leq
\|T_0-R_{\theta_n}T_0\|_{L^2(\mu)}^2
\lesssim
\theta_n^2,
\]
where the last inequality again uses boundedness of the support of $\nu_0$. This proves \eqref{eq:john-valid-upper}.

\subsection{Proof of Lemma~\ref{lem:assouad-map-hellinger}}
Fix an arbitrary estimator $\widehat T$. For each $j$, write
\[
\|f\|_{j}^{2}
:=
\int_{G_j}\|f(x)\|_2^2\,d\mu(x).
\]
Fix $j$ and fix the signs $\sigma_{-j}$ outside coordinate $j$. Let
\[
\sigma^+=(\sigma_{-j},+1),
\qquad
\sigma^-=(\sigma_{-j},-1).
\]
Define a test of the $j$th bit by nearest-neighbor decoding in the local
$L^2(G_j,\mu)$ distance:
\[
\widehat \sigma_j
\in
\argmin_{\tau\in\{\pm1\}}
\|\widehat T-T_{(\sigma_{-j},\tau)}\|_j^2 .
\]
Ties may be broken arbitrarily.

If the true sign is $+1$ and $\widehat \sigma_j=-1$, then
\[
\|\widehat T-T_{\sigma^-}\|_j
\leq
\|\widehat T-T_{\sigma^+}\|_j .
\]
Therefore, by the triangle inequality,
\[
\|T_{\sigma^+}-T_{\sigma^-}\|_j
\leq
\|T_{\sigma^+}-\widehat T\|_j
+
\|\widehat T-T_{\sigma^-}\|_j
\leq
2\|\widehat T-T_{\sigma^+}\|_j .
\]
Thus
\[
\|\widehat T-T_{\sigma^+}\|_j^2
\geq
\frac14
\|T_{\sigma^+}-T_{\sigma^-}\|_j^2
\geq
\frac{\Delta_j}{4}
\]
on the event $\{\widehat\sigma_j=-1\}$. The same argument with the signs
reversed gives
\[
\|\widehat T-T_{\sigma^-}\|_j^2
\geq
\frac{\Delta_j}{4}
\]
on the event $\{\widehat\sigma_j=+1\}$.

Hence
\begin{align*}
&\frac12
\mathbb E_{\sigma^+}
\|\widehat T-T_{\sigma^+}\|_j^2
+
\frac12
\mathbb E_{\sigma^-}
\|\widehat T-T_{\sigma^-}\|_j^2  \\
&\qquad\geq
\frac{\Delta_j}{4}
\left[
\frac12 P_{\sigma^+}^{\otimes n}(\widehat\sigma_j=-1)
+
\frac12 P_{\sigma^-}^{\otimes n}(\widehat\sigma_j=+1)
\right].
\end{align*}
The term in brackets is the average error probability of a test between
$P_{\sigma^+}^{\otimes n}$ and $P_{\sigma^-}^{\otimes n}$. For any two distributions $P,Q$,
the minimal average testing error under the uniform prior is $\frac12(1-\TV(P,Q)).$
Moreover, $\TV(P,Q)\leq \sqrt{1-\rho(P,Q)^2}.$
Using \eqref{eq:assouad-affinity}, we obtain
\[
\frac12 P^{\otimes n}_{\sigma^+}(\widehat\sigma_j=-1)
+
\frac12 P^{\otimes n}_{\sigma^-}(\widehat\sigma_j=+1)
\geq
\frac12\left(1-\sqrt{1-\rho_0^2}\right).
\]
Therefore, for every fixed $\sigma_{-j}$,
\begin{align}
\label{eq:assouad-local-risk}
\frac12
\mathbb E_{\sigma^+}
\|\widehat T-T_{\sigma^+}\|_j^2
+
\frac12
\mathbb E_{\sigma^-}
\|\widehat T-T_{\sigma^-}\|_j^2
\geq
\frac{\Delta_j}{8}
\left(1-\sqrt{1-\rho_0^2}\right).
\end{align}

Now average \eqref{eq:assouad-local-risk} over all choices of
$\sigma_{-j}$ and sum over $j$. Since the sets $G_1,\ldots,G_M$ are
disjoint,
\begin{align*}
2^{-M}
\sum_{\sigma\in\{\pm1\}^M}
\mathbb E_\sigma
\int
\|\widehat T(x)-T_\sigma(x)\|_2^2\,d\mu(x)
&\geq
\sum_{j=1}^M
2^{-M}
\sum_{\sigma\in\{\pm1\}^M}
\mathbb E_\sigma
\|\widehat T-T_\sigma\|_j^2 \\
&\geq
\frac{1-\sqrt{1-\rho_0^2}}{8}
\sum_{j=1}^M \Delta_j .
\end{align*}
The supremum over $\sigma$ is at least the average over $\sigma$, and the
display holds for every estimator $\widehat T$. Taking the infimum over
$\widehat T$ proves the result.

\subsection{Proof of Theorem~\ref{thm:exponential-separation}}

We first record a simple geometric fact which helps to identify the OT map
in our construction.

\begin{lemma}[Isolated blob matching]
\label{lem:isolated-blob-matching}
Let
\[
S_i=B(x_i,r),\qquad C_i=B(y_i,r),\qquad i=1,\ldots,N,
\]
be pairwise disjoint source and target balls. Suppose that $\mu$ assigns
mass $p_i$ uniformly to $S_i$ and $\nu$ assigns the same mass $p_i$
uniformly to $C_i$. Assume that the prescribed matching $S_i\mapsto C_i$
satisfies
\begin{align}
\label{eq:isolated-blob-margin}
\max_i \sup_{x\in S_i,\ y\in C_i}\|x-y\|_2
<
\min_{i\neq k}\inf_{x\in S_i,\ y\in C_k}\|x-y\|_2 .
\end{align}
Then every optimal coupling between $\mu$ and $\nu$ is supported on
\[
\bigcup_{i=1}^N S_i\times C_i.
\]
Consequently, if $\mu$ is absolutely continuous, the OT map sends
$S_i$ to $C_i$ for every $i$. Since $C_i$ is a translate of $S_i$, this
map is the translation
\[
T(x)=x+(y_i-x_i),\qquad x\in S_i.
\]
\end{lemma}

\begin{proof}
Let $\pi$ be an optimal coupling. Since optimal couplings for the squared
Euclidean cost are cyclically monotone, the support of $\pi$ is
$c$-cyclically monotone, where $c(x,y)=\|x-y\|_2^2$.

Suppose, for contradiction, that $\pi(S_i\times C_k)>0$ for some
$i\neq k$. Consider the directed graph on $\{1,\ldots,N\}$ in which we
draw an edge $i\to k$ whenever $\pi(S_i\times C_k)>0$. Because the
source and target masses of each component are the same, any off-diagonal
edge belongs to a directed cycle. Thus there exist distinct indices
$i_1,\ldots,i_\ell$ and points
\[
(x_q,y_q)\in \text{supp}(\pi)\cap (S_{i_q}\times C_{i_{q+1}}),
\qquad q=1,\ldots,\ell,
\]
where $i_{\ell+1}=i_1$.

By \eqref{eq:isolated-blob-margin}, each $x_q\in S_{i_q}$ is strictly
closer to every point of $C_{i_q}$ than to every point of
$C_{i_{q+1}}$. Since $y_{q-1}\in C_{i_q}$, we have
\[
\|x_q-y_q\|_2^2>\|x_q-y_{q-1}\|_2^2,
\]
where $y_0:=y_\ell$. Summing over $q$ gives
\[
\sum_{q=1}^\ell \|x_q-y_q\|_2^2
>
\sum_{q=1}^\ell \|x_q-y_{q-1}\|_2^2,
\]
which contradicts cyclic monotonicity. Therefore $\pi$ has no
off-diagonal mass.

It follows that the restriction of $\pi$ to each $S_i\times C_i$ is an
optimal coupling between the uniform measure on $S_i$ and the uniform
measure on $C_i$. Since $C_i=S_i+(y_i-x_i)$, the translation
$x\mapsto x+(y_i-x_i)$ is an optimal map. When the source is absolutely
continuous, the OT map is unique $\mu$-a.e., and hence the OT
map agrees with this translation on each $S_i$.
\end{proof}

We now verify \eqref{eq:isolated-blob-margin} for our
construction. At the level of centers, the assigned source-target
distance in each block is
\[
\|x_{j,L}-y_{j,L}^{\sigma}\|_2^2
=
\|x_{j,R}-y_{j,R}^{\sigma}\|_2^2
=
a^2+(\varepsilon-\theta_n)^2.
\]
The closest incorrectly assigned target center is the opposite target
center in the same block, for which the squared distance is $a^2+(\varepsilon+\theta_n)^2.$ 
Indeed, target centers in different blocks have horizontal distance at least $3\varepsilon-\theta_n$ from the source center under consideration, whereas the opposite target center in the same block has horizontal distance $\varepsilon+\theta_n$; since $\theta_n\ll \varepsilon$, the latter is the closest incorrect target center.
Thus the center-level distance gap is
\[
\sqrt{a^2+(\varepsilon+\theta_n)^2}
-
\sqrt{a^2+(\varepsilon-\theta_n)^2}
=
\frac{4\varepsilon\theta_n}{\sqrt{a^2+(\varepsilon+\theta_n)^2} + \sqrt{a^2+(\varepsilon-\theta_n)^2}} \gtrsim_a \varepsilon \theta_n.
\]
All centers remain in a fixed bounded set. Passing from centers to balls
of radius $r_n$ perturbs distances by at most $C r_n$, where
$C$ depends only on the diameter of this bounded set and on $a$. Since
$r_n\ll \varepsilon\theta_n$, the strict separation
\eqref{eq:isolated-blob-margin} holds for all sufficiently large $n$.
Therefore the OT map $T_\sigma$ from $\mu$ to $\nu_\sigma$ sends
the source ball around $x_{j,L}$ to the target ball around
$y_{j,L}^{\sigma}$ and the source ball around $x_{j,R}$ to the target
ball around $y_{j,R}^{\sigma}$, by translation.

\begin{lemma}
\label{lem:lower-bound-analysis}
There exists a constant $c_a>0$, depending only on $a$,
such that if
\[
r_n \leq c_a \varepsilon \theta_n,
\]
then, for all sufficiently large $n$, the following statements hold:
\begin{enumerate}

\item For every $\sigma\in\{\pm1\}^n$ and every $j=1,\ldots,n$,
\[
\int_{G_j}
\|T_\sigma(x)-T_{\sigma^{(j)}}(x)\|_2^2\,d\mu(x)
=
4a^2\varepsilon.
\]
Thus the local separation condition in Lemma~\ref{lem:assouad-map-hellinger}
holds with
\[
\Delta_j=4a^2\varepsilon.
\]

\item If $P_\sigma=\nu_\sigma$, then for every $\sigma$ and every $j$,
\[
\rho(P_\sigma^{\otimes n},P_{\sigma^{(j)}}^{\otimes n})
=
(1-\varepsilon)^n.
\]
In particular, since $\varepsilon=1/n$, there is a universal constant
$\rho_0>0$ such that
\[
\rho(P_\sigma^{\otimes n},P_{\sigma^{(j)}}^{\otimes n})
\geq \rho_0
\]
for all sufficiently large $n$.
\end{enumerate}
\end{lemma}
\begin{proof}
Let $\sigma^{(j)}$ be
obtained from $\sigma$ by flipping the $j$th sign. On the left source ball,
\[
y_{j,L}^{\sigma^{(j)}}-y_{j,L}^{\sigma}
=
(0,2\sigma_j a),
\]
and on the right source ball,
\[
y_{j,R}^{\sigma^{(j)}}-y_{j,R}^{\sigma}
=
(0,-2\sigma_j a).
\]
Thus, for every $x\in G_j$,
\[
\|T_\sigma(x)-T_{\sigma^{(j)}}(x)\|_2^2=4a^2.
\]
Since $\mu(G_j)=\varepsilon$, it follows that
\[
\int_{G_j}
\|T_\sigma(x)-T_{\sigma^{(j)}}(x)\|_2^2\,d\mu(x)
=
4a^2\varepsilon.
\]

It remains to verify the Hellinger affinity condition. The distributions
$\nu_\sigma$ and $\nu_{\sigma^{(j)}}$ agree on every block except block
$j$, whose total mass is $\varepsilon$. Inside block $j$, their supports
are disjoint for all sufficiently large $n$, since
$r_n\ll \varepsilon\theta_n\leq \theta_n$ and $a>0$ is fixed. Hence the
one-sample Hellinger affinity is exactly the common mass:
\[
\rho(\nu_\sigma,\nu_{\sigma^{(j)}})
=
1-\varepsilon.
\]
Hellinger affinity tensorizes under products, so
\[
\rho(\nu_\sigma^{\otimes n},\nu_{\sigma^{(j)}}^{\otimes n})
=
\rho(\nu_\sigma,\nu_{\sigma^{(j)}})^n
=
(1-\varepsilon)^n.
\]
Since $\varepsilon=1/n$,
\[
(1-\varepsilon)^n
=
\left(1-\frac1n\right)^n
\]
is bounded below by a positive universal constant for all sufficiently
large $n$. This proves the Hellinger affinity condition and completes the
proof.
\end{proof}

An immediate consequence of this result is that 
the family we constructed satisfies the assumptions of
Lemma~\ref{lem:assouad-map-hellinger} with $M=n$,
\[
\Delta_j=4a^2\varepsilon,
\qquad j=1,\ldots,n,
\]
and with a universal Hellinger affinity lower bound $\rho_0>0$, which in turn yields the lower bound on $\otminimaxrisk(\targetclass_n)$.

\subsubsection{Upper Bounds on the Valid Map Risk}
To conclude the proof of the theorem we need to upper bound $\validminimaxrisk(\targetclass_n)$.
Let $o = (1,\ldots,1)$ and use the deterministic estimator
\[
\widehat T:=T_o.
\]
Fix any true sign vector $\sigma$. We construct a valid map $S_\sigma$
from $\mu$ to $\nu_\sigma$ that stays close to $T_o$.

On block $j$, if $\sigma_j=1$, set $S_\sigma=T_o$. If
$\sigma_j=-1$, then $T_o$ sends the right source blob to the
upper target blob centered at $(u_j+\theta_n,a)$ and the left source blob
to the lower target blob centered at $(u_j-\theta_n,-a)$. Instead, define
$S_\sigma$ to send the right source blob to the upper target blob of
$\nu_\sigma$, centered at $(u_j-\theta_n,a)$, and the left source blob to
the lower target blob of $\nu_\sigma$, centered at
$(u_j+\theta_n,-a)$. This changes only the horizontal coordinate, by
$2\theta_n$.

Thus $S_{\sigma\#}\mu=\nu_\sigma$ and, on every source point,
\[
\|T_{o}(x)-S_\sigma(x)\|_2\leq 2\theta_n.
\]
Therefore
\[
\|T_{o}-S_\sigma\|_{L^2(\mu)}^2
\leq
4\theta_n^2.
\]
Since this holds for every $\sigma$, we have
\[
\validminimaxrisk(\targetclass_n)
\leq
4\theta_n^2
=
4e^{-2n},
\]
which proves \eqref{eq:exp-valid-upper}.

\end{document}